\documentclass{article}

\usepackage[margin=1in]{geometry}
\usepackage{bbm}
\usepackage{amsfonts, amssymb, amsthm, amsmath}
\usepackage{mathtools}
\usepackage{bm}
\usepackage{float}

\usepackage{tikz}
\newcommand{\eps}{\varepsilon}
\newcommand{\reals}{\mathbb{R}}

\newcommand{\dm}{d}
\newcommand{\lp}{\left}
\newcommand{\rp}{\right}
\newcommand{\E}{\textbf{E}}
\newcommand{\Prob}{\textbf{Pr}}

\newcommand{\func}{f}
\newcommand{\set}{\mathcal{I}}

\newcommand{\A}{{A}}

\newcommand\Dist{\mathcal{D}}

\newcommand\Mech{\mathcal{M}}

\newcommand\Exp{\mathbb{E}}

\newtheorem{claim}{Claim}
\newtheorem{lemma}{Lemma}
\newtheorem{theorem}{Theorem}
\usepackage[bottom]{footmisc}

\begin{document}

\title{Black-box Methods for Restoring Monotonicity}

\author {
Evangelia Gergatsouli \\
UW-Madison\\
\tt{gergatsouli@wisc.edu}
\and
Brendan Lucier\\
Microsoft Research \\
\tt{brlucier@microsoft.com}
\and
Christos Tzamos\\
UW-Madison\\
\tt{tzamos@wisc.edu}
}

\date{}
\maketitle

\begin{abstract}
		In many practical applications, heuristic or approximation algorithms
		are used to efficiently solve the task at hand. However their solutions
		frequently do not satisfy natural monotonicity properties of optimal
		solutions.  In this work we develop algorithms that are able to restore
		monotonicity in the parameters of interest.

		Specifically, given oracle access to a (possibly non-monotone)
		multi-dimensional real-valued function $\func$, we provide an algorithm
		that restores monotonicity while degrading the expected value of the
		function by at most $\eps$. The number of queries required is at most
		logarithmic in $1/\eps$ and exponential in the number of parameters. We
		also give a lower bound showing that this exponential dependence is
		necessary. 

		Finally, we obtain improved query complexity bounds for restoring the
		weaker property of $k$-marginal monotonicity. Under this property,
		every $k$-dimensional projection of the function $\func$ is required to
		be monotone. The query complexity we obtain only scales exponentially
		with $k$.
\end{abstract}


\section{Introduction}

``You are preparing a paper for an upcoming deadline and try to fit the content
within the page limit. You identified a redundant sentence and remove it but to
your surprise, the page count of your paper increases!''

This is an example where a natural monotonicity property expected from the
output fails to hold, resulting in unintuitive behavior. As another example,
consider the knapsack problem where given a collection of items with values and
weights the goal is to identify the most valuable subset of items with total
weight less than $W$. Now if the capacity of the knapsack increases, the new
set of items about to be selected are expected to be at least as valuable as
before. While such a monotonicity property holds under the optimal solution,
when heuristic or approximate algorithms are used, monotonicity can often fail.

In this work we develop tools to restore monotonicity in a black-box way.  Our
goal is to create a meta-algorithm that is guaranteed to be monotone, while
querying a provided (possibly non-monotone) algorithm behind the scenes.  For
example, such a meta-algorithm might query the provided oracle at many
different inputs in an attempt to smooth out non-monotonicities.


More precisely, we can describe the output of a (possibly non-monotone)
algorithm using a function  $f: \reals^\dm \rightarrow [0,1]$ that measures the
quality of the solution at any given point $x \in \reals^\dm$.  We assume that
inputs are drawn from a known product distribution, and that there is an
(unknown) feasibility condition being satisfied by the function $f$.  Since $f$
may not initially be monotone, we want to correct it through our meta-algorithm
while additionally maintaining the following three properties: it needs to be
query efficient, feasible and comparable to the original algorithm in terms of
expected solution quality.  To ensure feasibility, at any given point $x$ we
allow outputting any solution that corresponds to some smaller input $y \le x$,
thus achieving quality $f(y)$ at input $x$. The exact constraints of the
initial algorithm are unknown, and we have to infer them using $f$, in a
black-box way. We want our meta-algorithm to do this in a way that the
resulting quality function $\tilde f: \reals^\dm \rightarrow [0,1]$ is monotone
but also have expected quality $\Exp [\tilde {f}(x) ]$ that is comparable to
the expected quality of the original algorithm, $\Exp[ f(x)]$ with respect to
the input distribution. We finally require that our meta-algorithm is query
efficient, meaning that it does not require querying too many points of the
original function in order to return an answer for a given point.

One natural way of solving this problem is to pick for every $x$ the best
answer that corresponds to any input $y \le x$, resulting in $\tilde f(x) =
\max_{y \le x} f(y)$. While this ensures that solutions are monotone and that
the expected quality is always better than before, it requires a large number
of queries to identify the best possible answer among inputs $y \le x$. A more
query efficient strategy would be to always output a constant solution that
corresponds to the smallest possible solution. While this is monotone, feasible
and uses few queries, it has very low expected quality.

\subsection{Results}

Our main result is stated informally below:

\begin{theorem}
  There is a meta-algorithm $\Mech_f$ that corrects the monotonicity of any
given function $\func$ through queries. The meta-algorithm is feasible and has
expected quality loss $\eps$ under a given product distribution of inputs. The
total (expected) number of queries required for every input is at most
$O(\log(\dm/\eps))^\dm$.
\end{theorem}

Our meta-algorithm starts by discretizing the space to a fine grid of $d/\eps$ points. We show that
this step incurs at most an $\eps$ penalty in expected function value. Given this discretization
it is straightforward to correct monotonicity by querying all the points in the grid at a cost of $(\dm/\eps)^\dm$ 
and for any $x$ returning the best solution for any smaller input.

We remark that our meta-algorithm is significantly more efficient than this naive approach, achieving number of queries that is at most 
logarithmic in $1/\eps$. It is able to obtain this speedup by searching over a hierarchical
partition of the space to efficiently determine which value to assign to the given input at query time, rather
than preprocessing all the answers. Additionally,
our algorithm is \emph{local} and computes answers on the fly without requiring any
precomputation and without the need to remember any prior queries to answer the
future ones.

We also note that our algorithm depends exponentially on the number of input
parameters $\dm$. This means that while the algorithm is extremely efficient for
small $\dm$, the savings become less significant when $\dm$ grows. 
Such an exponential dependence in the number of parameters, however, is unavoidable for correcting
monotonicity. As we show, for any black-box scheme that aims to correct monotonicity, there are always
some problems where either it would fail to be feasible or monotone, or it would require exponentially 
many queries to calculate the appropriate answers.

\begin{theorem}\label{thm:lower_bound_prel}
  Let $\Mech$ be any feasible meta-algorithm for monotonicity with
query complexity $q = 2^{o(\dm)}$. Then, there exists an input 
function $\func: \{0,1\}^\dm \rightarrow \{0,1\}$ with $\Exp[\func(x)] \ge
1-2^{-\Omega(\dm)}$ such that with $\Exp [\Mech_\func(x)] \le 2^{-\Omega(\dm)}$.

\end{theorem}

Theorem~\ref{thm:lower_bound_prel} shows that ensuring monotonicity when
$\dm$ is large can be quite costly, either in queries or the quality of the
solution. To better understand this tradeoff, we consider a weakening of
monotonicity, namely $k$-marginal monotonicity, where we only require
monotonicity of the $k$-dimensional projections of the function to be monotone.
That is, when $k = 1$, we want that for any coordinate $i \in [\dm]$, the
function $\Exp[ \Mech_\func(x) | x_i ]$ to be monotone with respect to $x_i$.
For larger $k > 1$, we want that for any subset $\set \subseteq [\dm]$ of size
$|\set|\le k$, the function $\Exp[ \Mech_\func(x) | x_\set ]$ is monotone with
respect to $x_S$. For this setting, we show that:

\setcounter{theorem}{3}
\begin{theorem}
  There is a meta-algorithm $\Mech_f$ that corrects the $k$-marginal
monotonicity of any given function $\func$ through queries. The meta-algorithm
is feasible and has expected quality loss $\eps$ under a given product
distribution of inputs. The total (expected) number of queries required for
every input is at most $\lp( \frac{\dm}{\eps} \rp)^{O(k)}$.
\end{theorem}

Note that when requiring marginal monotonicity, the dependence on $d$ is
improved from exponential to polynomial. Instead, the query complexity is only
exponential in $k$.


\subsection{Related work}

A very related line of work to our paper considers the problem of online property-preserving data reconstruction.
In this framework introduced by Ailon et al.  \cite{AiloChazComaLiu2008} there is a set of unreliable data that should satisfy a certain
structural property, like monotonicity or convexity.  The reconstruction
algorithm acts as a filter for this data, such that whatever query the user
makes on them is answered in a way consistent with the property they should
satisfy.  Ailon at al. \cite{AiloChazComaLiu2008} proposed the first
algorithm for monotonicity reconstruction, and in the follow-up work, Saks and
Seshadhri \cite{SaksSesh2010} designed a more efficient and local algorithm for the same problem. 
The main focus of this work is to compute a function that is not different than the original
function in a large number of inputs. In comparison, we allow our algorithm to output arbitrary solutions
at any point subject to a feasibility criterion and consider the expected quality of the solution as a 
measure of performance.

In addition to these works on upper bounds on monotonicity reconstruction, Bhattacharyya et al. in \cite{BhatGrigJhaJungRask2012} proved a lower bound for
local monotonicity reconstruction of Saks and Seshadhri using
transitive-closure spanners. Several reconstruction algorithms have also been
proposed for reconstructing Lipschitzness \cite{JhaRask2013}, convexity
\cite{ChazSesh2011}, connectivity of directed or undirected graphs
\cite{CampGuoRubi2013}, or a given hypergraph property \cite{AustTao2010},
while \cite{ChakFiscMats2014} focuses on lower bounds.


Our work can also be viewed as a special case of the \emph{Local Computation
Algorithms} framework introduced by \cite{RubiTamiVardXie2011} and
\cite{AlonRubiVardXie2012}. In this model the algorithm is required to answer
queries that arrive online such that the answers given are always consistent
with a specific solution. The algorithms presented also do not need to remember
old answers to remain consistent, exactly as our algorithms do. Such local
algorithms have been designed for many problems like maximal independent set
maximal matching, approximate maximum matching, set cover, vertex coloring,
hypergraph coloring \cite{AlonRubiVardXie2012, MansVard2013, LeviRonRubi2014,
EvenMediRon2014, LeviRonRubi2016, LeviRubiYodp2017,
PartRubiVakiYodp2019,GhafUitt2019, GrunMitrRubi2020}.


Finally, our work is also closely related to the black-box reductions literature in algorithmic mechanism design, where
we are given access to an algorithmic solution - an oracle- and the goal
is to implement an incentive compatible mechanism with similar performance.
Ensuring incentive compatibility amounts to preserving a similar monotonicity property, like cyclic monotonicity.  This line of
work was initiated by Hartline and Lucier \cite{HartLuci2015}, and later
several reductions were given under different solution concepts, approximate or exact Bayesian Incentive Compatibility \cite{BeiHuan2011,
HartKleiMale2015, DughHartKleiNiaz2017, GergLuciTzam2019} along with a lower
bound for Dominant Strategy Incentive Compatibility \cite{ChawImmoLuci2012}.

\section{Preliminaries} \label{sec:prelims}

We are given oracle access to a function $\func:\reals^\dm \rightarrow [0,1]$,
and a stream of input points $x\in \reals^\dm$ for which we need to evaluate
$\func$.  Our goal is to give an answer $\Mech_\func$ for every point such that
satisfies some motion of monotonicity and it also has the following properties
\begin{enumerate}
	\item Feasibility: $\Mech_\func(x)\leq \max_{y\leq x} \func(y)$
	\item Close in expectation to the initial function: $\Exp [\Mech_\func(x))]
			\geq \Exp[\func(x)] - \eps$
\end{enumerate}

\paragraph{Evaluation}
We evaluate our algorithms on their query complexity. Query complexity is
defined as the maximum number of times we invoke the oracle $\func$ in order to give us
an answer and the goal is to invoke the oracle as few times as possible.


\paragraph{Distributional assumptions}
We assume that each coordinate $x_i$ of every point $x$ in the input sequence is drawn according
to a distribution $\Dist_i$. We denote the product of these $\Dist_i$ by
$\Dist\in \Delta(\reals^d)$.  We want $\Exp_{x\sim \Dist}[f(x)]$ to be close to
the expectation of the transformed $f$ which we denote by $\Mech_\func$. From
now on, we will omit the $x\sim \Dist$ when it is clear from the context.

\paragraph{Monotonicity} We proceed by defining the various notions of
monotonicity we will be using throughout this work. For $x,y\in \reals^d$ we
say that $x<y$ when $x_i\leq y_i$ for all $i\in[d]$. We say that the function
$\func$ is monotone if $x\leq y$ implies $f(x)\leq f(y)$ for every $x,y\in
\reals^\dm$.

A relaxation of the monotonicity requirement is that of \emph{marginal}
monotonicity. We say that $\func$ is \emph{marginally} monotone if $f_i(x_i)
\triangleq \Exp_{x_{-i} \sim \Dist}[f(x_i, x_{-i})]$ is monotone. Note that
this is weaker than monotonicity since even if all marginals are monotone, this
does not imply that $f$ is also monotone. 

A further relaxation of the marginal monotonicity is the
$k$-marginal monotonicity. Similarly to $k$-wise independent variables, we say
that a function is $k$-marginally monotone when for every $\set\subseteq
[d]$ such that $|\set|\leq k$ the function
\[
		f_{\set} (x_{\set}) \triangleq 
		\Exp_{x_{-\set}} [f_{\set}(x_{\set}, x_{-\set})] 
\]
is monotone, where we denote by $x_{\set}$ (or $x_{-\set}$) the size $k$-vector of
all the coordinates $i$ that also belong (or do not belong) to the set $\set$, and
by $f_{\set}(x_{\set})$ the marginal of the coordinates $i\in \set$ that gets
as input a vector of size $|\set|$ with the coordinates $x_{\set}$.

\paragraph{Discretization of the Domain} In all the algorithms described in
the following sections, we assume the domain is discrete. To do that we use the discretization
process described below. Intuitively, we split the domain into smaller intervals with equal
probability and then ``shift" every coordinate's distribution downward by
sampling a point from the lower interval. This is made more formal here, we
first describe the process for the single dimensional case ($d=1$), and then for general $d$.

We convert the oracle $\func$ to an oracle for a piecewise
constant function $\widetilde{\func}$ with $1/\eps$ pieces. The function 
$\widetilde{\func}$ is such that $\Exp [\widetilde{\func}(x)] \geq
\Exp[\func(x)] - \eps$, and  $\widetilde{\func}(x) \le \max_{y \le x}
\func(y)$.

For this purpose, we split the support of the distribution into $m = \frac 1
\eps$ intervals $I_1,...,I_m$ such that each has probability $\eps$, i.e. for
every $i\in [m]$, $\int_{x \in I_i} \Dist(x) = \eps$. Then, for each interval $I_i$
we draw a random $x_i$ according to the conditional distribution $\Dist | I_i$. For
any $x \in I_i$, we set 
\[
	\widetilde{\func}(x)=
		\begin{cases}
			0 & i=0\\
			\func(x_{i-1}) & i > 0
		\end{cases}.
\]
It is easy to see that $\widetilde{\func}(x) \le \max_{y \le x} \func(y)$ since
$\widetilde{\func}(x)$ is either $0$ or equal to $\func(x_{i-1})$ where $x_i \le x$.
To bound the expectation, we note that 
\[
\begin{aligned}
  \Exp_{x \sim \Dist} [\widetilde{\func}(x)] 
		&= \sum_{i=1}^{m-1} \eps \Exp_{x_i \sim \Dist | I_i} [{\func}(x_i)] \\
		& = \sum_{i=1}^{m} \eps  \Exp_{x_i \sim \Dist | I_i} [{\func}(x_i)] 
				- \eps  \Exp_{x_m \sim \Dist | I_m} [{\func}(x_m)]  \\
		&= \Exp_{x \sim \Dist} [\func(x)] - \eps  \Exp_{x_m \sim \Dist | I_m} [{\func}(x_m)] \\
		&\ge \Exp_{x \sim \Dist} [\func(x)] - \eps 
\end{aligned}
\]

For $d>1$ dimensions the above process is essentially the same for every
coordinate with the difference that now we choose $m=\dm/\eps$ points and every
interval $I_{ik}$ for $k\in [m]$ of coordinate $i$ has probability $\eps/\dm$.
Let the input vector be $\bm{y}= (y_1, y_2, \ldots, y_\dm)$.  For every
coordinate $y_i \in I_{ik}$, we draw a random $z_i$ from the conditional
$\Dist|I_{i(k-1}$ and finally evaluate $\func$ at these points $\bm{z}= (z_1,
z_2, \ldots, z_\dm)$.  Note that this is feasible, similarly to before, as each
input returns an outcome generated by $\func$ on a pointwise smaller input.
This perturbation effectively shifts each coordinate's distribution downward,
removing the range $I_{im}$, therefore removing from the expectation any input
that had at least one coordinate in the last interval. This occurs with
probability at most $1 - \lp(1 - \frac{\eps}{\dm}\rp)^\dm \leq \eps$, and
therefore the new expectation is not $\eps$ far from the old.

Assuming a discrete domain, where every coordinate $x_i\in \{w_{i1}, w_{i2},
\ldots , w_{im}\}$, we define the low $1$-neighborhood of a point as the set of
points that are smaller in exactly $1$ coordinate. Formally, the low
$1$-neighborhood of a point $y$ is $\mathcal{N}(y)=\{x: \exists! k\in [m]
\text{ such that } x_k = w_{ik} \text{ and } y_k = w_{k(i+1)}\text{ for some
}i\in[m-1] \}$.


\setcounter{theorem}{0}

\section{Monotonicity}\label{sec:monot}
In this section we show how to design an algorithm to correct a non monotone
function while preserving its expectation, using
$O\lp(\log{\frac{\dm}{\eps}}\rp)^\dm$ queries.

\begin{theorem}\label{thm:monotonicity}
There exists a meta-algotithm $\Mech_f$ that corrects the monotonicity for any
function $\func: \reals^\dm \rightarrow [0,1]$. $\Mech_\func$ is \emph{feasible}, has
query complexity $O(\log{\frac{\dm}{\eps}})^\dm$, and has $\Exp
[\Mech_\func(x))] \geq \Exp[\func(x)] - \eps$, where the first expectation is
taken over the input distribution and the randomness in the meta-algorithm.
\end{theorem}

To establish Theorem~\ref{thm:monotonicity} we first show how to solve the
problem for the single-dimensional case, $\dm = 1$.  Observe initially that the
trivial meta-algorithm that simply queries the function $\func$ at every point in the
(sufficiently discretized) domain has complexity $O(1/\epsilon)$.  To get the
improvement to $O(\log(1/\epsilon))$ we first consider the following ``greedy''
meta-algorithm: query the (discretized) domain in a uniformly random order, and
for each point $x$ sequentially, define the transformed output $\Mech_\func(x)$
to be the value closest to $\func(x)$ that maintains monotonicity with the
values of $\Mech_\func$ constructed so far.  We prove that this random
meta-algorithm preserves the function value in expectation over the random
query order and then that this transformation can be implemented locally, with
only logarithmically many queries per input. The details of these steps are
presented in subsection~\ref{sec:single} below.

\subsection{Single-dimensional case ($\dm=1$)}
\label{sec:single}
Using the discretization process described in section~\ref{sec:prelims}, we may
assume that $\func$ is supported on $m$ points and the underlying distribution
is uniform. We now provide an algorithm to obtain a monotone function $\func'$
that has the same expectation as $\func$ and performs $O(\log
\frac{m}{\delta})$ queries to $\func$ with probability $1-\delta$. We later
choose $\delta = \eps$ and cap the queries to $O(\log \frac{m}{\eps})$ and show
that it is possible to do this while maintaining feasibility, monotonicity and
incurring error at most $\eps$. 

\paragraph{Construction of the Oracle} The function $\func'$ that the algorithm
outputs is a function $\Mech^{\pi}_{\func}$ for a uniformly random permutation $\pi$.
Given a permutation $\pi$ of $[m]$, we define the function
$\Mech^{\pi}_{\func}: [m]\rightarrow [0,1]$ by setting for every $i \in [m]$, 

\[
	\Mech^{\pi}_{\func}( \pi_i ) = 
	\begin{cases}
		H_i & \func(\pi_i) > H_i\\
		L_i & \func(\pi_i) < L_i\\
		\func(\pi_i) & \text{otherwise,}
	\end{cases} 
\]

where $H_i$ is the value $\Mech^{\pi}_{\func}( \pi_j )$ of the lowest point $
\pi_j > \pi_i$ with $j<i$ or $\infty$ if such a point does not exist.
Similarly, $L_i$ is the value $\Mech^{\pi}_{\func}( \pi_j )$ placed at the
highest point $ \pi_j < \pi_i$ with $j<i$ or $-\infty$ if such a point does not
exist.  That is, the function $\Mech^{\pi}_{\func}$ visits all points according
to the permutation $\pi$ and greedily assigns values consistent with the
choices made for the previous points visited so far to preserve monotonicity.
Equivalently, one can write that $\Mech^{\pi}_{\func}( \pi_i ) =
\text{median}\{\func(\pi_i), L_i, H_i\}$.

We now show that this choice of $\func'$ satisfies the properties of
Theorem~\ref{thm:monotonicity} with the following two claims. Their proofs are
deferred to section~\ref{app:monot} of the Appendix.

\begin{claim}\label{cl:monot_d1_feasibility}
  For any permutation $\pi$, the function $\Mech^{\pi}_{\func}$ is monotone and
  satisfies for all $x \in [m]$, $\Mech^{\pi}_{\func}(x) \le \max_{y \le x}
  \func(y)$.
\end{claim}

\begin{claim}\label{cl:expectation}
  $\E_{x\sim \mathbb{U}\lp([m]\rp)}[ \func(x) ] 
  	= \E_{x\sim \mathbb{U}\lp([m]\rp)}[ \func'(x) ]$.
\end{claim}

It remains to show that one can evaluate $\Mech^\pi_\func$ at any point $x \in
[m]$ without querying the oracle for $\func$ at all points. To do this, we make
the following observation: once we have computed $\Mech^\pi_f(\pi_1)$, in order
to calculate the values of $\Mech^\pi_\func(y)$ for any $y>\pi_1$, we don't
need to know the values of $\Mech^\pi_\func(z)$ or of $\func(z)$ at any
$z<\pi_1$. Similarly, to calculate the values of $\Mech^\pi_\func(y)$ for any
$y<\pi_1$, we don't need to know the values of $\Mech^\pi_\func(z)$ or of
$\func(z)$ at any $z>\pi_1$.

Building on this idea, we use the following algorithm to evaluate $\Mech^\pi_f(x)$. 

\paragraph{Description of Oracle for $\Mech^\pi_f$} At any point in time $i$,
we keep track of a range of relevant points $\{l_i,...,r_i\}$ starting with
$l_1=0$ and $r_0=m$.  We also keep track of a lower bound, $\text{lb}_i$, and
an upper bound, $\text{ub}_i$, on the value of $\Mech^{\pi}_\func(x)$ starting
with $\text{lb}_0 = -\infty$ and $\text{ub}_0 = +\infty$.

For any $i \in [m]$, if $\pi_i \in \{l_{i-1},...,r_{i-1}\}$, then it is
relevant and must be evaluated.  Its value is then according to the definition
$\Mech^{\pi}_\func(\pi_i) = \text{median}\{ \func(\pi_i), H_i, L_i \}$, which
is equal to $\text{median}\{ \func(\pi_i), \text{ub}_{i-1}, \text{lb}_{i-1} \}$
for the upper and lower bounds we have so far. Once the value is computed, if
$\pi_i > x$, we set $\text{ub}_{i} = \Mech^{\pi}_\func(\pi_i)$ while keeping
$\text{lb}_{i} = \text{lb}_{i-1}$ and the relevant interval now becomes
$\{l_{i},...,r_{i}\} = \{l_{i-1},...,\pi_i-1\}$. Similarly, if $\pi_i < x$, we
update $\text{lb}_{i} = \Mech^{\pi}_\func(\pi_i)$, $\text{ub}_{i} =
\text{ub}_{i-1}$ and set $l_{i}=\pi_i+1$ and $r_{i}=r_{i-1}$.  In contrast, if
$\pi_i \not \in \{l_{i-1},...,r_{i-1}\}$, it is irrelevant and is not
evaluated. The interval and upper and lower-bounds are then kept the same for
the next iteration.

The following claim shows that for a uniformly random permutation $\pi$, the
above process only queries the oracle $\func$ at a few points. The proof is
deferred to section~\ref{app:monot} of the appendix.

\begin{claim}\label{cl:monot_queries}
With probability $1-\delta$, the oracle $\func'$ can be evaluated at any point $x
\in [m]$ using at most $O(\log{\frac{m}{\delta}})$ queries 
  to oracle $\func$.
\end{claim}

We now argue that it is possible to perform the transformation so that the algorithm
always makes at most $O(\log \frac{m}{\eps})$ while maintaining feasibility,
monotonicity and incurring error at most $\eps$. Our construction of the oracle
maintains for every interval a high and a low value that points in the interval
may take. The interval shrinks with good probability by a constant factor at
every step which gives the high probability result. To ensure that no more than
$O(\log \frac{m}{\eps})$, we need to ensure that every interval shrinks at most
$O(\log \frac{m}{\eps})$ times. Indeed, we can enforce that after these many
rounds, if the interval has not shrunk to a single point every point in the
interval is allocated the lowest value. This ensures monotonicity while it
incurs a decrease in the expected value.  As this decrease happens with
probability at most $\eps$ and the decrease is bounded by $1$ the total error
is at most $\eps$.

\subsection{Extending to many dimensions}
In order to generalize to many dimensions, we apply our construction for
this ``single-dimensional case'' to fix monotonicity in each direction
separately starting with the first. The key property we use is that when given
oracle access to a function that is monotone in the first $i-1$ coordinates,
our construction of the meta-algorithm will fix the monotonicity in the $i$-th coordinate \emph{while
preserving monotonicity in the $i-1$ first coordinates}. This allows us to
obtain a chain of oracles $\func=\func_0,\func_1,...,\func_n=\func'$ where
$\func_i$ is monotone in the first $i$ coordinates. Evaluating $\func_i$
requires only $O(\log{\frac{\dm}{\eps}})$ queries to oracle $\func_{i-1}$ and
gets error at most $\eps/\dm$.  Thus, to evaluate $\func'=\func_n$,
$O(\log{\frac{\dm}{ \eps}})^\dm$ queries to oracle $\func$ are sufficient to get
error $\eps$.  Details are deferred to section~\ref{sec:multiple} of the
appendix.


\section{Lower bound}\label{sec:LB}
Having designed the meta-algorithm to ``monotonize" a function, in this section we
show that the exponential dependence on the dimension our previous algorithm
exhibits, as shown in Theorem~\ref{thm:monotonicity}, is actually necessary
even when the domain is the boolean hypercube $\{0,1\}^\dm$ and the
distribution $\Dist$ of values is uniform.  The idea for this lower bound is to show
that there exists a function such that any monotone and feasible meta-algorithm
$\Mech$ with subexponential query complexity $q = 2^{o(\dm)}$ will have very
low expectation.  This is made formal in Theorem~\ref{thm:lower_bound} below.

\begin{theorem}\label{thm:lower_bound}
		Let $\Mech$ be any feasible meta-algorithm that fixes monotonicity with
query complexity $q = 2^{o(\dm)}$. Then, there exists an input 
function $\func: \{0,1\}^\dm \rightarrow \{0,1\}$ with $\Exp[\func(x)] \ge
1-2^{-\Omega(\dm)}$ such that with $\Exp [\Mech_\func(x)] \le 2^{-\Omega(\dm)}$.

If the meta-algorithm is infeasible or non-monotone with probability $\delta$, then
$\Exp[ \Mech_\func(x)] \le 2 \delta + 2^{-\Omega(\dm)}$.
\end{theorem}

To prove the statement, we construct a distribution over functions on $\{0,
1\}^\dm$ and show that any monotone and feasible meta-algorithm must
have very low expectation with high probability. 

We consider the following family of functions parametrized by $(z,S,T)$ where
$z \in \{0, 1\}$ and $S,T \subseteq [\dm]$ so that $S \subseteq T$. We define for
any $X \subseteq [\dm]$
{\small 
\begin{align*}
\func^{z}_{S,T}(X)= 
\begin{dcases}
	0\ & |X|< \frac{4\dm}{10}\\
	\begin{drcases}
			0\ &  X \subseteq T \text{ and } |X \setminus S|> \frac{\dm}{10}\\
			z\ &  X \subseteq T \text{ and } |X \setminus S| \leq \frac{\dm}{10} \\
			1\ &  X \not\subseteq T 
	\end{drcases}  & |X|\geq \frac{4\dm}{10}
\end{dcases}
\end{align*}
}

We also define the function 
\[
	\func^{1}(X) = \mathbb{I}_{\lp\{ |X|\geq \frac{4\dm}{10}\rp\} }
\]

Observe that even though these functions are defined over subsets of $[\dm]$ it
is straightforward to view each of these subsets as a point in $\{0,1\}^\dm$.

We define a distribution over the family of functions by selecting $z$
uniformly at random. The sets $S$ and $T$ will also be random variables with
$S$ including each element with probability $1/2$ and $T$ including each
element with probability $3/4$.  Since $S \subseteq T$, this means that given
$S$, the set $T$ contains each element outside of $S$ with probability $1/2$. Similarly
given $T$, the set $S$ contains each element of $T$ with probability $2/3$. We also
define random variable $X$ that is a uniformly random subset of $[\dm]$.

What we are trying to achieve with these functions is while $\func^1$ has high
expectation, the function $\func^z_{S,T}$ does not, but we will not be able to
tell them apart. The idea is that in both functions $f^z_{S,T}$ there is the
``low" set $T$ where the function outputs $0$, but inside it there is a hidden
set $S$ where the function is either $1$ or $0$. The two claims shown below,
will prove that first we cannot distinguish between the function that gives $S$
either $0$ or $1$, and then this function from the ``high" function that gives
$1$ to all large inputs. This is shown in figure~\ref{fig:lowerbound}
below\footnote{The figure serves as a simplified example of the structure of
the functions, the sizes of the sets are not on scale}.

\begin{figure*}[t]
	\centering
	 \begin{tikzpicture}[scale=0.75]
\tikzset{square/.style={black, thick}}
\tikzset{squareSM/.style={blue, thick}}
\tikzset{areaStyle/.style={draw=none, fill=gray!30, opacity=0.5}}

\tikzset{rx/.style={x radius=#1},ry/.style={y radius=#1}}

\pgfmathsetmacro{\length}{3.5}
\pgfmathsetmacro{\height}{5.5}
\pgfmathsetmacro{\rx}{2}
\pgfmathsetmacro{\ry}{5}

\def\lab{{"$f^1$","$f^0_{S,T}$","$f^1_{S,T}$"}}

\def\xcoords{{0, \length+1,	2*\length + 2}}
\def\ycoords{{0, 		 0,			   0 }}

\node[] at (\xcoords[0] - 1.6, \ycoords[0] + 0.9*\height) {\small \textcolor{blue}{$|X|<4\dm/10$}};
\node[] at (\xcoords[0] - 1.6, \ycoords[0] + 0.4*\height) {\small $|X|\geq 4\dm/10$};

\foreach \i in {1,2,3} 
{ 
	\pgfmathsetmacro{\labelz}{\lab[\i-1]}
	\pgfmathsetmacro{\x}{\xcoords[\i-1]}
	\pgfmathsetmacro{\y}{\ycoords[\i-1]}

	\draw[square] (\x, \y) rectangle (\x + \length,\y + 0.8*\height) node[] at (\x + \length/2,\y - 0.5){\labelz};
	\draw[areaStyle] (\x, \y) rectangle (\x + \length,\y + 0.8*\height);
	
	\draw[squareSM] (\x, \y+ 0.8*\height) rectangle (\x + \length,\y + \height);
}


\pgfmathsetmacro{\x}{\xcoords[1]}
\pgfmathsetmacro{\y}{\ycoords[1]}
\draw[fill=white] (\x + 0.5*\length, \y+0.4*\height) circle [rx= 0.8cm, ry=1.5cm] node[black, above=1cm]{$T$};
\draw[fill=white] (\x + 0.5*\length, \y+0.4*\height) circle [rx= 0.5cm, ry=0.7cm] node[black]{$S$};

\pgfmathsetmacro{\x}{\xcoords[2]}
\pgfmathsetmacro{\y}{\ycoords[2]}
\draw[fill=white] (\x + 0.5*\length, \y+0.4*\height) circle [rx= 0.8cm, ry=1.5cm] node[black, above=1cm]{$T$};
\draw[fill=gray!30] (\x + 0.5*\length, \y+0.4*\height) circle [rx= 0.5cm, ry=0.7cm] node[black]{$S$};

\end{tikzpicture}
	 \caption{Output for each function. The gray colour means the output is
	 $1$, white means $0$.}
		\label{fig:lowerbound}
\end{figure*}
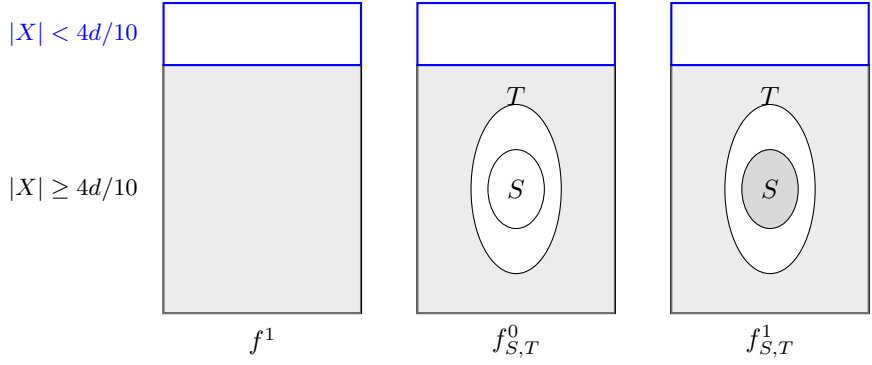

\begin{claim}\label{cl:M_is_1_fixed_T}
	$\Pr \lp[ \Mech_{\func^1_{S,T}}(T) \neq \Mech_{\func^0_{S,T}}(T)  \rp] \leq
	q 2^{-\frac {\dm}{450}}$
\end{claim}

\begin{claim}\label{cl:Dtv}
	$\Pr \lp[ \Mech_{\func^1_{S,T}}(S) \neq \Mech_{\func^1}(S)  \rp] \leq q 2^{-\frac {\dm}{10}}$
\end{claim}

The proofs of both the claims are deferred to section~\ref{app:lb} of the
appendix. It is now easy to complete the proof of Theorem~\ref{thm:lower_bound}
by setting $\func= \func^1$.
\begin{align*}
		\Exp [\Mech_{\func^1}(x)]  &= \Exp {\lp[ \Mech_{\func^1}(S) \rp] } \\
		& \leq  \Exp {\lp[ \Mech_{\func^1_{S,T}}(S) \rp] } + q 2^{-\frac {\dm}{10}} \\
		& \leq  \delta + \Exp {\lp[ \Mech_{\func^1_{S,T}}(T) \rp] } + q 2^{-\frac {\dm}{10}} \\
		&\le  \delta + \Exp {\lp[ \Mech_{\func^0_{S,T}}(T) \rp] } + q 2^{-\frac {\dm}{10}} + q 2^{-\frac {\dm}{450}}  \\
		&\le  2 \delta + q 2^{-\frac {\dm}{10}} + q 2^{-\frac {\dm}{450}}  \\
		&= 2 \delta + 2^{-\Omega(\dm)}
\end{align*}

where the first line follows since $S$ is chosen uniformly at random, then we
use Claim~\ref{cl:Dtv} and then we use that $f$ satisfies monotonicity with
probability $1-\delta$. Following this, the third line follows from
Claim~\ref{cl:M_is_1_fixed_T} and then we use the fact that
$\Mech_{\func^0_{S,T}}(T) \leq \max_{Y \subseteq T} \func^0_{S,T}(Y) = 0$  with
probability $1-\delta$. Finally we get the result for any $q = 2^{o(\dm)}$.

In contrast, for the initial function $\func^1$ we get $\Exp {\lp[ {\func^1}(X)
\rp] } = \Pr {\lp[ |X| \ge \frac{4\dm}{10} \rp] } \ge 1 - 2^{-\Omega(\dm)} $.

\section{Marginal Monotonicity}\label{sec:marginal}

In this section, we switch gears towards the relaxations of monotonicity
defined in section~\ref{sec:prelims}. We start by considering \emph{marginal
monotonicity}. In this case, we want to guarantee that each of the marginals of
the function will be monotone, and not loosing much in expectation. As it turns
out, we can achieve this in time polynomial in $\dm/\eps$. The formal statement follows.

\begin{theorem}\label{thm:monotonicity.bic}
		There exists a feasible meta-algorithm $\Mech_\func$
		that fixes marginal monotonicity for any function
		$f:\reals^d\rightarrow [0,1]$. $\Mech_\func$ is \emph{feasible}, has
		query complexity $O\lp(\text{poly}\lp( \frac{\dm}{\eps}\rp) \rp)$, and satisfies
		$\Exp[ \Mech_\func(x) ] \geq \Exp [\func] -  \eps$, where the first
		expectation is taken over the input distribution and the randomness of
		the meta-algorithm.
\end{theorem}

The discretization process described in section~\ref{sec:prelims} will also be
used here. Therefore we can safely assume that the domain is discretized and
supported on $m$ different values which we denote by $w_{i1} < \dotsc < w_{im}$.

We start by assuming we are given query access directly to the marginal
distribution $f_i$ in every dimension and showing that we can achieve the
theorem using $O(\dm m)$ queries. Then in
subsection~\ref{subsec:estimate_marginals}, we show how to achieve the same
result by only querying the initial function $\func$ in order to estimate the
marginals.

\subsection{Transformation Using Exact Marginals}\label{subsec:known_marginals}
In this case we assume that we know exactly each one of the marginals $f_i(x_i)$.

Consider meta-algorithms of the following form: in each dimension $i$ there
will be a mapping $\phi_i \colon \reals \to \reals$, with $\phi_i(x_i) \le x_i$
for all $x_i \in \reals$.  We will write $\phi(x) = (\phi_1(x_1), \dotsc,
\phi_\dm(x_\dm))$.  We will then define $\func'(x) = \func(\phi(x))$.  Observe that
any such $\func'$ satisfies feasibility, since $\func'(x) = \func(\phi(x)) \le
\max_{y \le x} \func(y)$.

We will build the mapping $\phi$ iteratively, starting with the identity
mapping, which we will call $\phi^0$.  Since the distribution over values is
discrete, it suffices to define each mapping $\phi_i$ on the finitely many
values in the support of the distribution (for each one of the $d$ dimensions).

Suppose our current mapping is $\phi^r = (\phi^r_1, \dotsc, \phi^r_\dm)$, for
some $r \geq 0$ and let $\func^r(x) = \func(\phi^r(x))$.  If $\func^r_i$, the
$i$'th marginal function is monotone for every $i$, then we will choose $\func'
= \func^r$.  

Otherwise, there is some $i$ and some $j < m$ such that $\func^r_i(w_{ij}) >
\func^r_i(w_{i(j+1)})$.  In this case, we will define $\phi^{r+1}$ as follows:
$\phi^{r+1}(w_{i(j+1)}) = \phi^r(w_{ij})$, and $\phi^{r+1} = \phi^r$ on all
other inputs.  That is, whenever $\func'$ is given $w_{i(j+1)}$ as input, we
will instead invoke $\func^r$ as though we have gotten $w_{ij}$.  Observe that
this modification chains: if on some previous iteration we had mapped input
$w_{ij}$ to $w_{i(j-1)}$, then after this iteration we will ultimately be
passing the input $w_{i(j-1)}$ to the original function $\func$.  As argued
above, this modified function $\func^{r+1}$ will be feasible.  Moreover, 
\begin{align*}
		\E_{x \sim \Dist}[\func^{r+1}(x)]  = \E_{x_i}[\func^{r+1}_i(x_i)] 
										   \geq \E_{x_i}[\func_i^r(x_i)] 
										  = \E_{x \sim \Dist}[\func^r(x)]	
\end{align*}
so $\func^{r+1}$ has only weakly greater expected value than $\func^r$.

We can think of $\func^{r+1}$ as acting on a reduced domain, where the possible
input $w_{i(j+1)}$ is removed from the support of $x_i$'s distribution and
instead its probability mass is added to that of some lower value,
$\phi^r(w_{ij})$.  Under this interpretation, each iteration reduces the total
number of possible input values in the support of $\Dist$ by $1$.  This process
must therefore stop at or before iteration $r = \dm (m-1)$, since a marginal over
a single input value is always monotone.  Thus, after at most $\dm(m - 1)$
iterations, this process will terminate at an function $\func'$ that is
feasible, monotone, and has $\E_{x \sim \Dist}[\func'(x)] \geq \E_{x \sim
\Dist}[\func(x)]$.

\subsection{Sampling to Estimate Marginals}\label{subsec:estimate_marginals}

The meta-algorithm above assumes direct access to the marginal distributions
even after the modifications we make at each step. We will show how to remove
these assumptions, at the cost of a loss of $\eps$ on the expectation of
$\func'$.  This $\eps$ loss is due to sampling error, and can be made as small
as desired with additional sampling.

Prior to viewing the input, our meta-algorithm will estimate each one of the
marginals.  For each $i\in [\dm]$, take $O(\log(\dm m) /\delta^2)$ samples
$x_{-i}$ and observe $\func(x_i, x_{-i})$. Let $\tilde{\func}_i(x_i)$ for the
empirical mean of the observed samples.  By Hoeffding inequality and a union
bound over all coordinates and all possible input values, we will have that
$|\tilde{\func}_i(x_i) - \func_i(x_i)| \leq \delta$ for all $i$ and all $x_i$,
with high probability.  

We will then apply our meta-algorithm from above using the marginals
$\tilde{\func}$ as an oracle.  This generates mappings $\phi_i$, such that the
new ``monotonized" marginals $ \tilde{\func}_i(\phi_i)$ have weakly increased
expectation relative to $\tilde{\func}_i$. If $|\E[\tilde{\func}_i(x_i)] -
\func_i(x_i)| \leq \delta$, then we also have $|\E[\tilde{\func}_i(\phi(x_i))]
- \func_i(\phi(x_i))| \leq \delta$ for all $i$ and all $x_i$ as well.
Monotonicity of $\tilde{\func}_i(\phi_i)$ therefore implies
$(2\delta)$-approximate monotonicity of $\func(\phi)$, and that
$\E[\func(\phi(x))] \geq \E[\func(x)] + 2\delta$.

\paragraph{From Approximate to Exact Monotonicity}

From the above steps, we can assume access to a function $\func$ that is 
$(2\delta)$-approximately monotone.  We will implement a new function $\func'$
such that, on input $x$, say with $x_i = w_{ij_{(i)}}$ for each $i$, returns
$\max\{0, \func(x) - 2\delta \cdot \sum_i (m - j_{(i)})\}$.  Then $\func'$ is
monotone, as whenever $x > y$ we have that either $\func'(y) = 0$ or $\func'(x)
- \func'(y) \geq \func(x) - \func(y) + 2\delta \geq 0$.  Moreover, this
modification reduces the expected allocation of $\func$ by at most
$3m\dm\cdot \delta$.  So as long as $\delta \leq \eps / (\dm m)$, its expected
allocation is within $\eps$ of $\func$.


\section{$k$-Marginal Monotonicity}\label{sec:gen_marginal}

Having designed an algorithm for marginal monotonicity, we move on to a
generalization that guarantees $k$-\emph{marginal
monotonicity}. In this case, given oracle access to a function $f:\reals^\dm
\rightarrow [0,1]$ we want to guarantee that all $k$-marginals of $f$ will be
monotone. 

\begin{theorem}\label{thm:monotonicity_k_bic}
	There exists a meta-algorithm $\Mech_\func$ to fix $k$-marginal
	monotonicity (whp) for any function $f:\reals^\dm \rightarrow [0,1]$.
	$\Mech_\func$ is \emph{feasible}, has query complexity $\tilde{O} \lp( \frac{\dm^{2k+6}}{\eps^{2k+3}} \rp)$
	and satisfies $\Exp[ \Mech_\func(x) ] \geq \Exp [\func] - \eps$, where the
	first expectation is taken over the input distribution and the randomness
	of the meta-algorithm.  
\end{theorem}

In order to prove the theorem, we start by using the discretization method
described in section~\ref{sec:prelims}. After having a discrete domain, our
meta-algorithm estimates in each step all\footnote{Note that there is no way to
		avoid fixing all the $d \choose k$ marginals; even ${d\choose k} - 1 $
monotone marginals cannot guarantee that the ${d\choose k}$'th is also
monotone.  } the $d^k$ marginals and whenever it detects a non-monotonicity in
one of them, it fixes it within $\delta$ by defining a set of replacement
rules. Then for an input $x$, we sequentially try to apply the rules, bu
starting from the first, and trying to match the given point with one of the
patterns of the rules.  After we applied the possible rules, we reached some
other point $x'$ and then we output $f(x')$. An example of the set of
replacement rules is shown in Table~\ref{table:replacements}.

\begin{table}[t]
		\centering
		\begin{tabular}{ccll}
              & \textbf{Initial}            &               & \multicolumn{1}{c}{\textbf{Replacement}}  \\ \cline{2-4}
		\multicolumn{1}{c|}{\textbf{Rule 1}}      & $001?1?0?$     & $\rightarrow$ & \multicolumn{1}{l|}{$001?0?0?$} \\ \cline{2-4}
		\multicolumn{1}{c|}{\textbf{Rule 2}}      & $0?1?01?1$ & $\rightarrow$ & \multicolumn{1}{l|}{$0?1?01?0$}   \\ \cline{2-4}
		\multicolumn{1}{c|}{$\ldots$}             &                    &  $\ldots$              & \multicolumn{1}{l|}{}                     \\ \cline{2-4}
		\multicolumn{1}{c|}{\textbf{Rule $\ell$}} & $?01?0?10$          & $\rightarrow$ & \multicolumn{1}{l|}{$?01?0?00$}  \\ \cline{2-4}
		\end{tabular}
	\caption{Example of replacement rule list. For any input $x$, the input is sequentially transformed to a different one by applying the rules from top to bottom, the function value at the resulting vector is then returned.}
	\label{table:replacements}
\end{table}

We describe the algorithm in three steps; first we describe the replacement
process, in order to fix the non-monotonicities, given that we have access to
the exact margninals, and then describe the sampling process used to estimate
the marginals. Using the first two steps, we are guaranteed only approximate
monotonicity, so as the third step we show how to further modify our function
to achieve exact monotonicity.



\subsection{Transformation Using Exact Marginals}
Assuming now that we have access to all the $k$-marginals exactly, we describe
an iterative process, in order to correct the non-monotonicities in all the
$k$-marginals. The answer we give is the transformed function $f'$.

As a direct extension from the previous marginals case, consider transformations of the
form: $\phi:\reals^\dm \rightarrow \reals^\dm$ such that $\phi(x) \leq x$, for
all $x\in \reals^d$. We define $f'(x) = f(\phi(x))$.

We denote by
$\phi_{\set}(x_{\set}):\reals^k\rightarrow \reals^\dm$, for some $\set
\subseteq [\dm]$, with $|\set|\leq k$, the projection of the function $\phi$ in the
$k$ coordinates that are in $\set$, where the variables $x_j$ for
$j\not\in\set$ are treated as constants and remain the same. Recall that by $x_{\set}$ we
denote all the coordinates $x_i$ such that $i$ is in the set $\set\subseteq
[d]$ with $|\set|\leq k$. 

Intuitively, what the process does is when we detect a non-monotonicity between
some input $x$ and its neighbor $y$, where $y$ is larger in at least one
coordinate\footnote{We can assume without loss of generality that $y$ is higher
in \emph{exactly} one coordinate.}, from then on we always map $y$ to $x$. This
is reflected in the function $\phi$ that replaces $y$ with $x$  to correct the
non-monotonicities. Since this process is done iteratively, when we map $y$ to
$x$ and $x$ to some other input $z$, it means $y$ is ultimately mapped to $z$.

More formally, we start from the identity function $\phi^0(x) = x$, and
iteratively define $\phi^1, \phi^2, \ldots$.  In this case we define a
non-monotonicity as the case when there exists a set $\set\subseteq [d]$, with
$|\set|\leq k$, such that $x_{\set} < y_{\set}$ we have that $f_{\set}(x_{\set})>
f_{\set}(y_{\set})+\delta$. Observe that in this case, we only ensure that the
function is $\delta$ monotone. In iteration $r$, one of the following cases can
happen

\begin{enumerate}
	\item There is no such set $\set$: we terminate the process and return $f'(x)=f(\phi^r(x))$
	\item There exists such a set $\set$: we set
\[ 
	\phi^{r+1}_{\set}(y_{\set}) = \max_{z\in\mathcal{N}(y)} \phi^r_{\set}(z_{\set})
\]
where recall from section~\ref{sec:prelims}
that $\mathcal{N}(y)$ is the low $1$-neighborhood of $y$.
\end{enumerate}

Using this process, and exactly how we argued in the previous section, the
output function $f'$ is feasible and that $\Exp_{x\sim\Dist}[f^{r+1}(x)] \geq
\Exp_{x\sim\Dist}[f^r(x)]$.

Observe that every time the replacement described above happens, the
expectation of the transformed function $f'$ increases by at least
$\delta/m^k$, meaning that this process cannot happen more than $m^k/\delta$
times in total since $f'\leq 1$. When the process halts, the transformed
function $f'$ is feasible and approximately $\delta$-monotone, with
$\Exp_{x\sim\Dist}[f'(x)] \geq \Exp_{x\sim\Dist}[f(x)]$.

\subsection{Sampling to Estimate Marginals}
In the process described above, we assumed that the exact marginals were known.
In reality, in every step we estimate the marginals again before checking for
non-monotonicities, by sampling points $x\in\reals^d$ and observing $f(x)$.

This step differs from the estimation step in the marginals case in that we do
not draw samples $x_{\set}$ from each specific marginal $f_{\set}$,
but directly from the function $f$ and then we estimate each marginal. 



Recall from the discretization process, that now the distribution over the
$m^k$ different values $f$ can take is uniform, which means that by drawing
samples from $f$, we need $m^k$ samples in expectation to get a sample from a
specific marginal. Using this fact, Hoeffding's inequality and a union bound
over all different $d^k$ marginals, $m^k$ values and all $m^k/\delta$ rounds
this sampling process is happening, we need $\frac{k}{m^k \delta^2} \log
(\frac{m^{2}d}{\delta})$ samples.

\subsection{From Approximate to Exact Monotonicity}\label{subsec:approx_monot}
This part is exactly the same as the previous section with the only difference
that we have access to a function that is $4\delta$-approximately monotone.
This difference is due to the fact that we only guaranteed $\delta$
monotonicity when we knew the exact marginals compared to exact monotonicity.
Therefore,  we now return a new function $\func'$ that on input $x$ with $x_i =
w_{ij_{(i)}}$ for each $i$, returns $\max\{0, \func(x) - 4\delta \cdot \sum_i
(m - j_{(i)})\}$, which as before is guaranteed to be monotone.

This modification reduces the expected allocation of $\func$ by at most
$4m\dm\cdot \delta$ in this case, so when $\delta \leq \eps / (\dm m)$, its
expected allocation is within $\eps$ of $\func$.

\paragraph{Query Complexity}
In order to calculate the query complexity of the meta-algorithm, recall that
there are $m^k/\delta$ rounds, where for each round we use $ k m^k/\delta^2
\log(m^2\dm/\delta)$ queries for the marginal estimation.

Using that $\delta < 1/(dm)$ and that $m=\dm/\eps$ from the discretization
process, we get that the query complexity is $\frac{k \dm^{2k+6}}{\eps^{2k+3}}
\log \lp( \frac{\dm^5}{\eps^3} \rp) = \tilde{O} \lp( \frac{\dm^{2k+6}}{\eps^{2k+3}} \rp)$.


\bibliography{bibnew}
\bibliographystyle{plain}

\clearpage
\appendix

\section{ Missing proofs from Section~\ref{sec:monot}}
\subsection{Single-dimensional case}
\label{app:monot}

\begin{proof}[Proof of Claim~\ref{cl:monot_d1_feasibility}]
  It is easy to see that the function is monotone by induction. Assuming that
		for any $j,k \le i-1$, if $\pi_j \le \pi_k$ then
		$\Mech^{\pi}_{\func}(\pi_j) \le \Mech^{\pi}_{\func}(\pi_k)$, we show that
		this property also holds for any $j,k \le i$. Indeed, by definition
		$\Mech^{\pi}_{\func}( \pi_i ) \le H_i \le \Mech^{\pi}_{\func}(\pi_j )$
		for any $\pi_j \ge \pi_i$  with $j < i$. Similarly,
		$\Mech^{\pi}_{\func}( \pi_i ) \ge L_i \ge \Mech^{\pi}_{\func}( \pi_j )$
		for any $\pi_j \le \pi_i$ with $j < i$. 
  
  To see that $\Mech^{\pi}_{\func}(x) \le \max_{y \le x} \func(y)$ for any $x \in [m]$,
notice that if $\func(\pi_i) \ge \Mech^{\pi}_{\func}( \pi_i )$ this is trivially true. We
thus only need to argue that this is true when $\Mech^{\pi}_{\func}( \pi_i ) = L_i$.
In this case, there is some $j<i$ with $\pi_j < \pi_i$, such that $L_i =
\Mech^{\pi}_{\func}( \pi_j )$. Again by induction, $\Mech^{\pi}_{\func}( \pi_j ) \le \max_{y
\le \pi_j} \func(y)$ and thus $L_i  \le \max_{y \le \pi_j} \func(y)  \le \max_{y \le
\pi_i} \func(y)$.
\end{proof}

\begin{proof}[Proof of Claim~\ref{cl:expectation}]

We first argue that it suffices to show the statement for functions $\func$ that
take values in $\{0,1\}$ instead of $[0,1]$.

For a function $g$, denote by $[g > t]$ the indicator function that $g(x) > t$.
One can check that $[\Mech^{\pi}_{\func} > t] = \Mech^{\pi}_{[\func>t]}$ as the definition of
$\Mech^{\pi}_{\func}$ only involves comparisons. Since the value $\Mech^{\pi}_{\func}(x) =
\int_0^1 [\Mech^{\pi}_{\func} > t](x) dt$, we get that $\Mech^{\pi}_{\func}(x) = \int_0^1
\Mech^{\pi}_{[\func>t]}(x) dt$.

We have $\Exp_{x\sim \mathbb{U}\lp([m]\rp)}[ \func'(x) ]  
		= \Exp_{x\sim \mathbb{U}\lp([m]\rp)}[ \Exp_{\pi}[ \Mech^{\pi}_{\func} ]] 
		= \int_0^1 \Exp_{x\sim \mathbb{U}\lp([m]\rp)}[ \Exp_{\pi}[ \Mech^{\pi}_{[\func>t]}(x) ]] dt$. 
Therefore if  for any boolean-valued function $g:[m]\rightarrow \{0,1\}$ we show that 
$\Exp_{x\sim \mathbb{U}\lp([m]\rp)}[ \Exp_{\pi}[ \Mech^{\pi}_{g}(x) ]] 
	= \Exp_{x\sim \mathbb{U}\lp([m]\rp)}[ g(x) ]$ 
then we obtain the required statement, as we can use $g = [\func>t]$ for any $t \in [0,1]$.

We move on to prove the statement for functions $\func$ taking values in
$\{0,1\}$. The description of the constructed $\func'$ becomes much simpler in
this case.
Starting from the interval $\{1,...,m\}$ the algorithm first selects an element
$i$ uniformly at random. If $\func(i)=0$, it sets $\func'(j)=0$ for any $j\le i$ and
recursively solves the problem in the interval $\{i+1,...,m\}$. If $\func(i)=1$,
it sets $\func'(j)=1$ for any $j\ge i$ and recursively solves the problem in the
interval $\{1,...,i-1\}$.

Let $\{l,...,r\}$ be the interval at the current iteration. Let $V(l,r)$ be the
value of this set defined as
\[ 
	V(l,r) = \sum_{x = l}^r \func(x) - \sum_{x=l}^r \Exp[\func'(x)]
\]
where the expectation is taken over the randomness of $\func'$ on the interval
$\{l,...,r\}$.

We will prove by induction on $r-l$ that $V(l,r) = 0$. In the case that
$r=l$ we will select the only point $x_l$ with probability 1, so $V(l,r) =
\func(x)-\func(x)=0$.  We assume that $V(l,r)=0$ for any $l,r$ with $r-l \le m-1$. 

Let $y \sim U(\{l,...,r\})$ be the uniformly random chosen point from
$\{l,...,r\}$. We distinguish two cases for $\func(y)$. 
If $\func(y)=0$ we obtain
\begin{align*}
  \sum_{x = l}^r \func(x) - \sum_{x = l}^r \Exp[\func'(x)|y] = \sum\limits_{x = l}^y \func(x) 
  		= \sum\limits_{x = l}^r \mathbbm{1}_{\{\func(x)=0, \func(y)=1, x \ge y\}}
\end{align*}
where the first equality follows, since conditional on $y$, $\func'(x)=0$ for $x\in
\{l,...,y\}$ and by the induction hypothesis $\sum_{x = y+1}^r \func(x)  - \sum_{x
= y+1}^r \Exp[\func'(x)|y] = 0$. Similarly, if $\func(y)=1$ we obtain
\begin{align*}
	  \sum_{x = l}^r \func(x) - \sum_{x = l}^r \E[\func'(x)|y] 
		= \sum\limits_{x = y+1}^r (\func(x)-1) 
		= - \sum\limits_{x = l}^r \mathbbm{1}_{\{\func(x)=1, \func(y)=0, x \le y\}}
\end{align*}

Overall, we have
\begin{align*}
	V(l,r) &= \frac{1}{r-l+1}  \sum_{y = l}^r  \left( \sum_{x = l}^r \func(x) - \sum_{x=l}^r  \Exp[\func'(x) | y] \right)\\
		&= \frac{1}{r-l+1} \sum_{y = l}^r  \left( \sum\limits_{x = l}^r \mathbbm{1}_{\{\func(x)=0, \func(y)=1, x \ge y\}}  -  \sum\limits_{x = l}^r \mathbbm{1}_{\{\func(x)=1, \func(y)=0, x \le y\}} \right) = 0
\end{align*}

The intuition behind this fact is that the expected loss incurred by turning
1's into 0's is exactly balanced by the expected gain by turning 0's to 1's.
\end{proof}

\begin{proof}[Proof of Claim~\ref{cl:monot_queries}]
  Fix a point $x \in [m]$ and a random permutation $\pi$.
  
  The oracle for $\Mech^\pi_f$ keeps track of an interval $\{l_i, ..., r_i\}$ and
makes a query only when the next point in the permutation lies in this
interval. As the permutation is chosen uniformly at random, the next point is
chosen uniformly in $\{l_i, ..., r_i\}$ and lies in the smaller interval
$\{\frac{3l_i+r_i}4, ..., \frac{l_i + 3 r_i}4\}$ with probability $1/2$. Every
time this happens, the algorithm discards at least $\frac{r_i - l_i}4$ of the
elements. As this shrinks the interval by a constant factor, it can happen at
most $O(\log m)$ times. By Hoeffding's inequality, the probability that after
$O(\log m) + \sqrt{ O(\log m \cdot \log (1/\delta)) }$ iterations the interval
size is still greater than 1 is at most $\delta$.  
  Since $O(\log m) + \sqrt{ O(\log m \cdot \log (1/\delta)) } = O(\log
(m/\delta)) $ we get that the number of oracle queries is at most $O(\log
(m/\delta))$ to evaluate $\A'(x)$ with probability $1-\delta$. To get the
required bound for every $x \in [m]$, we set $\delta \rightarrow \delta/m$ and
take a union bound on the probabilities of error. This only increases the
number of oracle queries by a constant factor, so the bound of $O(\log
(m/\delta))$ is still accurate.
\end{proof}

\subsection{Extending to many dimensions (general $\dm$)}
\label{sec:multiple}

To establish the result of Theorem~\ref{thm:monotonicity}, we now extend our
construction to the more general case with $\dm \geq 1$.  We apply our
single-dimensional construction from Section~\ref{sec:single} to fix
monotonicity in each direction separately starting with the first. 

We set $\func_0 = \func$. For every $i\in [\dm]$, based on the function
$\func_{i-1}$ that is monotone in the first $i-1$ coordinates we obtain a
function $\func_i$ that is monotone in the first $i$ coordinates. To do this we
apply our single dimensional construction at every single-dimensional slice
$\func_{i-1}(\cdot,x_{-i})$ of $\func_{i-1}$ for all choices $x_{-i} \in
\reals^{\dm}$ of the coordinates other than $i$.  Importantly, we use the same
randomness at every slice, for the choices of the points in the intervals
$I_1,...,I_{m}$ when performing the discretization to $m=\frac{1}{\eps}$ points
as well as for the chosen permutation $\pi$ over the discrete domain $[m]$.
This allows us to fix the monotonicity in coordinate $i$ while preserving
monotonicity in the first $i-1$ coordinates. It is easy to see that the
discretization preserves the monotonicity. We now argue that using the same
permutation for every slice also maintains the monotonicity.  The following
lemma shows that any two functions where one is smaller than the other,
preserve the same ordering after their monotonization.

\begin{lemma}\label{lem:M_monot_on_f}
Let  $f,g : [m] \rightarrow [0,1]$ such that $f(x)\leq g(x)$, for all $x \in
[m]$. For any permutation $\pi$, it holds that $\Mech^{\pi}_f(x) \leq
\Mech^{\pi}_g(x)$, for all $x\in [m]$.
\end{lemma}

\begin{proof}
As argued in Claim~\ref{cl:expectation}, it suffices to show the statement for
boolean valued functions $f,g:[m]\rightarrow \{0,1\}$. 
Let $i$ be the first point where $\Mech^{\pi}_f(\pi_i) \neq \Mech^{\pi}_g(\pi_i)$. By
the definition of $\Mech$, $H^{(g)}_i = H^{(f)}_i$ and $L^{(g)}_i = L^{(f)}_i$ and
thus it must be that $\Mech^{\pi}_f(\pi_i) = 0$ and $\Mech^{\pi}_g(\pi_i) = 1$. By
monotonicity  $\Mech^{\pi}_f(x) = 0$ for all $x \le \pi_i$ and $\Mech^{\pi}_g(x) = 1$
for all $x \ge \pi_i$. Thus, $\Mech^{\pi}_f(x) \leq \Mech^{\pi}_g(x)$, for all $x\in
[m]$.
\end{proof}

This allows us to obtain a chain of oracles $\func =
\func_0,\func_1,...,\func_\dm=\func'$ where $\func_i$ is monotone in the first
$i$ coordinates. Evaluating $\func_i$ requires only $O(\log{\frac{\dm}{\eps}})$
queries to oracle $\func_{i-1}$ and gets error at most $\eps/\dm$. Thus, to
evaluate $\func'=\func_\dm$, at most $O\lp(\log{\frac{\dm}{\eps}}\rp)^\dm$
queries to oracle $\func$ are required for error $\eps$.

\section{ Missing proofs from Section~\ref{sec:LB}}\label{app:lb}
\begin{proof}[Proof of Claim~\ref{cl:M_is_1_fixed_T}]
  This claim follows as with high probability, the transformation
  $\Mech_{\func^1_{S,T}}(T)$ cannot distinguish between $\func^0_{S,T}$ and
  $\func^1_{S,T}$.  To see this, let $R_1,...,R_q$ be the queries performed by
  $\Mech_{\func^0_{S,T}}(T)$. If all of those queries satisfy
  $\func^0_{S,T}(R_i) = \func^1_{S,T}(R_i)$ then it must be that
  $\Mech_{\func^1_{S,T}}(T) = \Mech_{\func^0_{S,T}}(T)$.
  
  Initially observe that for any $R$ where either $R \not \subseteq T$ or  $|R|
  < \frac {4\dm}{10} $ we have that $\func^0_{S,T}(R) = \func^1_{S,T}(R)$.  On
  the contrary, given the set $T$, for any $R \subseteq T$ such that $|R|\geq
  \frac {4\dm}{10}$, the functions $\func^1_{S,T}$ and $\func^0_{S,T}$ differ
  only when $|R\setminus S| \leq \dm/10$, therefore
  \[
		  \Prob[\func^0_{S,T}(R) \neq  \func^1_{S,T}(R) ] 
		  	= \Prob\lp[|R \setminus S| \leq \frac{\dm}{10} \rp]
  \]
  
  Since $S$ is a random subset of $T$ excluding each element independently with
  probability $1/3$, we get that the expected value of $\Exp[ |R \setminus S| ]
  = |R|/3 \ge \frac {4\dm}{30}$. By the Hoeffding inequality we get
  \[
	 \Pr\lp[ |R \setminus S| \le \Exp[ |R \setminus S| ] - \frac {\dm}{30} \rp] \le \exp\lp\{ -\frac {\dm}{450} \rp\}
  \]
  
    The claim then follows by a union bound on all $q$ queries.
\end{proof}

\begin{proof}[Proof of Claim~\ref{cl:Dtv}]
  Let $R_1,...,R_q$ be the queries performed by $\Mech_{\func^1_{S,T}}(S)$. If all
  of those queries satisfy $\func^1_{S,T}(R_i) = \func^1(R_i)$ then it must be that,
  $\Mech_{\func^1_{S,T}}(S) = \Mech_{\func^1}(S)$.
  
  For any query $R$, in order for $\func^1_{S,T}(R) \neq \func^1(R)$, it must be that
  $R \subseteq T$, $|R| \ge \frac {4\dm}{10}$ and $|R \setminus S| >
  \frac{\dm}{10}$.
  
  Given the set $S$, for any $|R| \ge \frac {4\dm}{10}$ and $|R \setminus S| >
  \frac{\dm}{10}$, we have $\Pr[ R \subseteq T ] = \Pr[ (R\setminus S) \subseteq
  T ] < 2^{-\frac{\dm}{10}}$ since given the set $S$, the set $T$ is created by
  including each coordinate in $[\dm]\setminus S$ with probability $1/2$.  The
  claim then again follows by a union bound on all $q$ queries.
\end{proof}

\end{document}